\documentclass[10pt,twocolumn,letterpaper]{article}

\usepackage{cvpr}              
\usepackage[hang,flushmargin]{footmisc}

\usepackage[dvipsnames]{xcolor}

\newif\ifdraft
\drafttrue

\ifdraft

\newcommand{\kac}[1]{{\color{orange}[\textbf{Kfir:} 
\textit{#1}]}}

\else
\newcommand{\pcc}[1]{}
\newcommand{\kac}[1]{}

\fi

\usepackage{multirow}
\usepackage{amsmath}
\usepackage{amsthm}
\usepackage[makeroom]{cancel}

\newtheorem{theorem}{Theorem}[section]
\newtheorem{corollary}{Corollary}[theorem]

\usepackage{makecell}
\usepackage{indentfirst}
\usepackage[dvipsnames]{xcolor}

\usepackage{booktabs,arydshln}
\usepackage{amssymb}
\usepackage{pifont}
\newcommand{\cmark}{\color{LimeGreen}{\ding{51}}}%
\newcommand{\xmark}{\color{red}{\ding{55}}}%

\makeatletter
\def\adl@drawiv#1#2#3{%
        \hskip.5\tabcolsep
        \xleaders#3{#2.5\@tempdimb #1{1}#2.5\@tempdimb}%
                #2\z@ plus1fil minus1fil\relax
        \hskip.5\tabcolsep}
\newcommand{\cdashlinelr}[1]{%
  \noalign{\vskip\aboverulesep
           \global\let\@dashdrawstore\adl@draw
           \global\let\adl@draw\adl@drawiv}
  \cdashline{#1}
  \noalign{\global\let\adl@draw\@dashdrawstore
           \vskip\belowrulesep}}
\makeatother
\def\RA{\rlap{\scalebox{1.3}{\;$\rightarrow$}}}

\def\uu{\mathbf{u}}
\def\vv{\mathbf{v}}

\def\AA{\mathbf{A}}
\def\BB{\mathbf{B}}

\def\nN{\mathcal{N}}

\def\Ee{\mathbb{E}}

\def\Re{\mathbb{R}}

\def\latex/{\LaTeX}
\def\bibtex/{\hologo{BibTeX}}

\newcommand{\RN}[1]{%
  \textup{\uppercase\expandafter{\romannumeral#1}}%
}

\definecolor{cvprblue}{rgb}{0.21,0.49,0.74}
\usepackage[pagebackref,breaklinks,colorlinks,citecolor=cvprblue]{hyperref}

\def\paperID{7665} 
\def\confName{CVPR}
\def\confYear{2024}

\title{Orthogonal Adaptation for Modular Customization of Diffusion Models}

\author{Ryan Po\\
Stanford University
\and
Guandao Yang\\
Stanford University
\and
Kfir Aberman\\
Snap Research
\and
Gordon Wetzstein\\
Stanford University
}

\begin{document}
\twocolumn[{%
\renewcommand\twocolumn[1][]{#1}%
\vspace{-1em}
\maketitle
\vspace{-3.2em}
\begin{center}
    \centering
    \captionsetup{type=figure}
    \includegraphics[width=\linewidth]{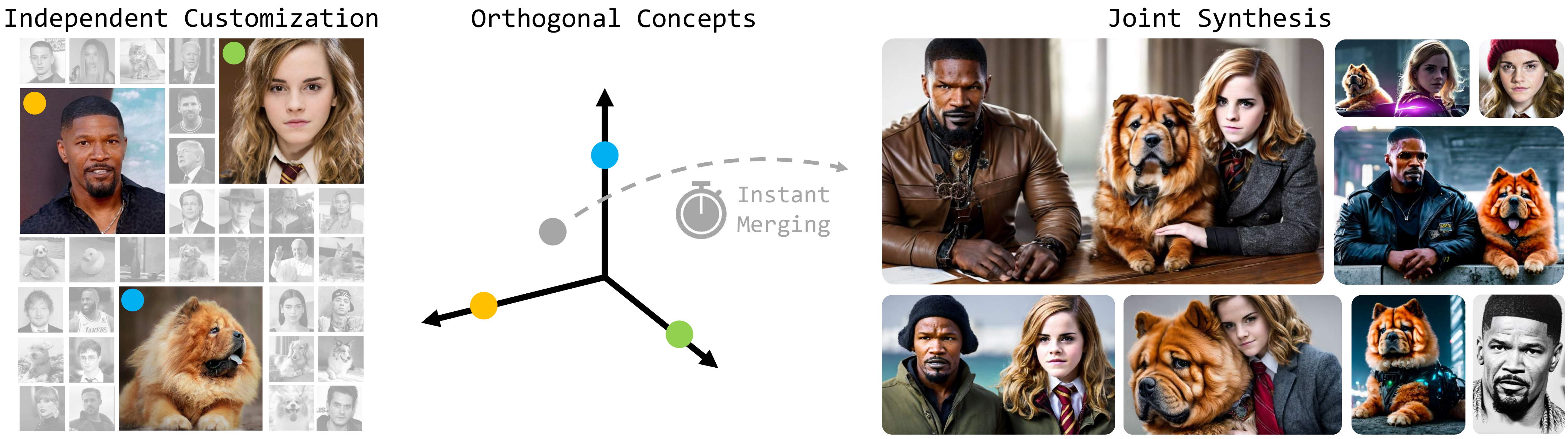}
    \vspace{-2em}
    \captionof{figure}{\textbf{Modular Customization of Diffusion Models.} Given a large set of individual concepts (left), the goal of \textit{Modular Customization} is to enable independent customization (fine-tuning) per concept, while efficiently merging a subset of customized models during inference, so that the corresponding concepts can be jointly synthesized without compromising fidelity. To tackle this, we propose \textit{Orthogonal Adaptation}, which encourages customized weights of one concept to be orthogonal to the customized weights of others.
    }
\end{center}%
\vspace{-0.2em}
}]
\renewcommand{\footnotesize}{\fontsize{7pt}{9pt}\selectfont}

\begin{abstract}
\vspace{-15pt}
Customization techniques for text-to-image models have paved the way for a wide range of previously unattainable applications, enabling the generation of specific concepts across diverse contexts and styles. While existing methods facilitate high-fidelity customization for individual concepts or a limited, pre-defined set of them, they fall short of achieving scalability, where a single model can seamlessly render countless concepts. In this paper, we address a new problem called Modular Customization, with the goal of efficiently merging customized models that were fine-tuned independently for individual concepts. This allows the merged model to jointly synthesize concepts in one image without compromising fidelity or incurring any additional computational costs.
To address this problem, we introduce Orthogonal Adaptation, a method designed to encourage the customized models, which do not have access to each other during fine-tuning, to have orthogonal residual weights. This ensures that during inference time, the customized models can be summed with minimal interference. 
\newcommand\blfootnote[1]{%
  \begingroup
  \renewcommand\thefootnote{}\footnote{#1}%
  \addtocounter{footnote}{-1}%
  \endgroup
}
\blfootnote{Project: {\color{magenta}\url{ryanpo.com/ortha}}; Demo: {\color{magenta}\url{hf.co/spaces/ujin-song/ortha}}}
Our proposed method is both simple and versatile, applicable to nearly all optimizable weights in the model architecture. Through an extensive set of quantitative and qualitative evaluations, our method consistently outperforms relevant baselines in terms of efficiency and identity preservation, demonstrating a significant leap toward scalable customization of diffusion models.

\end{abstract}    
\section{Introduction}
\label{sec:intro}

Diffusion models (DMs) mark a paradigm shift for computer vision and beyond. DM-based foundation models for text-to-image, video, or 3D generation enable users to create and edit content with unprecedented quality and diversity using intuitive text prompts~\cite{Po:2023:star_diffusion_models}. Although these foundation models are trained on a massive amount of data, in order to synthesize user-specific concepts (such as a pet, an item, or a person) with a high fidelity, they often need to be fine-tuned.

Several recent approaches to customizing DMs to individual concepts have demonstrated high-quality results~\cite{hu2022lora,gal2022image,ruiz2023dreambooth,kumari2022customdiffusion,voynov2023p+}. A multi-concept DM strategy, however, where several pre-trained concepts are mixed in a single image, remains challenging. Existing multi-concept methods~\cite{kumari2022customdiffusion,gu2023mix} either show a degradation in the quality of individual concepts when merged or require access to multiple concepts during training. The latter makes the process unscalable and raises privacy concerns when the different concepts belong to different users. Furthermore, in all cases the mixing process is computationally inefficient.
\begin{figure*}[t!]
    \centering
    \includegraphics[width=\linewidth]{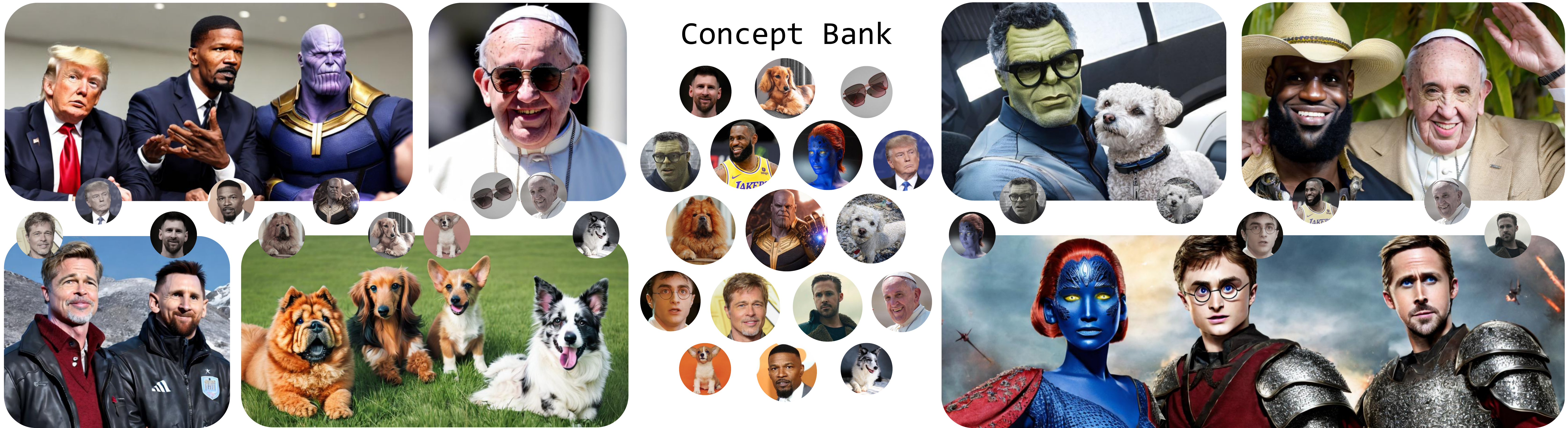}
    \vspace{-1.8em}
    \caption{\textbf{Gallery of multi-concept generations.} Our method enables efficient merging of individually fine-tuned concepts for modular, efficient multi-concept customization of text-to-image diffusion models. Each concept shown above was fine-tuned individually using orthogonal adaptation. Fine-tuned weight residuals are then merged via summation, enabling multi-concept generation.}
    \label{fig:gallery}
    \vspace{-1em}
\end{figure*}

We introduce orthogonal adaptation as a new approach to enabling instantaneous multi-concept customization of DMs. The primary insight of our work is that changing how the DM is fine-tuned for novel concepts can lead to very efficient mixing of these concepts. Specifically, we represent each new concept using a basis that is approximately orthogonal to the basis of other concepts. These bases do not need to be know a priori and different concepts can be trained independently of each other. A key advantage of our approach is that our model does not need to be re-trained when mixing several of our orthogonal concepts together, for example to jointly synthesize different concepts that were never seen together in any training example. Importantly, our approach is \emph{\bf{modular}} in that it enables individual concepts to be learned independently and in parallel without knowledge of each other. Moreover, it is privacy aware in the sense that it never requires access to the training images of concepts to mix them.

Consider a social media platform where millions of users fine-tune a DM using their personal concepts and want to mix them with their friends' concepts on their phones. Efficiency of the customization and mixing processes as well as data privacy are key challenges in this scenario. Our method addresses precisely these issues. A core technical contribution of our work is a modular customization and scalable multi-concept merging approach that offers better quality in terms of identity preservation than baselines at similar speeds, or similar quality to state-of-the-art baselines at significantly lower processing times.

\section{Related Work}
\label{sec:related}
\paragraph*{Text-conditioned image synthesis.}
The field of text-conditioned image synthesis has experienced significant advancements, driven by developments in GANs~\cite{goodfellow2020generative, brock2018large,karras2021alias,karras2019style,karras2020analyzing} and diffusion models~\cite{Ho2020DenoisingDP, Song2020DenoisingDI, Dhariwal2021DiffusionMB, Ho2022ClassifierFreeDG, Rombach2021HighResolutionIS, Pandey2022DiffuseVAEEC,nichol2021improved}.
Earlier efforts focus on applying GANs to various conditional synthesis tasks, including class-conditioned image generation~\cite{brock2018large,karras2019style,huang2022multimodal} and text-driven editing~\cite{abdal2022clip2stylegan,bau2021paint,gal2021stylegan,mokady2022self,xia2021tedigan,parmar2022spatially,roich2022pivotal}. 
More recently, the focus has shifted to large text-to-image models~\cite{Rombach2021HighResolutionIS,saharia2022photorealistic,yu2022scaling,ramesh2022hierarchical} trained on large-scale datasets~\cite{schuhmann2022laion}.
In this paper, we will utilize the open-source StableDiffusion~\cite{Rombach2021HighResolutionIS} architecture and build on its pre-trained checkpoints by fine-tuning. 

\addtolength{\tabcolsep}{-1pt}   
\begin{table}[t!]
\footnotesize
  \centering
  \begin{tabular}{lccc}
  \toprule
    Method & \shortstack{Fidelity \\\scriptsize{(Single-concept)}} & \shortstack{Efficient \\ Merging} & \shortstack{Fidelity \\ 
    \scriptsize{(Multi-concept)}} \\
    \midrule
    TI~\cite{gal2022image} & \xmark & \cmark & \xmark\\
    DB-LoRA\footnotemark~\cite{ruiz2023dreambooth} 
     & \cmark & \cmark & \xmark\\
    Custom Diffusion~\cite{kumari2022customdiffusion}
     & \xmark & \cmark & \xmark\\
    Mix-of-Show~\cite{gu2023mix}
     & \cmark & \xmark & \cmark\\
    \textbf{Ours} 
     & \cmark & \cmark & \cmark\\
    
    \bottomrule
  \end{tabular}
    \vspace{-0.8em}
    \caption{\textbf{Comparison of Solutions to Modular Customization.} Our customization approach excels in three key areas: (1) preserving the identity of individual concepts with high fidelity, (2) efficiently merging independently customized models, and (3) maintaining high concept fidelity for multi-concept image synthesis using the merged model. 
    }
  \label{tab:related_work}
  \vspace{-1.3em}
\end{table}
\footnotetext{assuming DB-LoRA fine-tuned models are merged with FedAvg~\cite{mcmahan2023communicationefficient}}
\addtolength{\tabcolsep}{1pt}   
\vspace{-1em}

\vspace{-1em}
\paragraph*{Customization.} 
The task of customization aims at capturing a user-defined concept, to be used for generation under various contexts. Seminal works such as Textual Inversion (TI)~\cite{gal2022image} and DreamBooth~\cite{ruiz2023dreambooth} tackle the problem of customization by taking a handful of images of the same concept to produce a representation of the subject to be used for controlled generation. TI captures new concepts by optimizing a text embedding to reconstruct target images using the conventional diffusion loss. Follow-up works, such as $\mathcal{P}+$~\cite{hertz2022prompt}, extend Texture Inversion with a more expressive token representation, improving generation subject alignment/fidelity. 
DreamBooth~\cite{ruiz2023dreambooth}, on the other hand, picks an uncommon word token and fine-tunes the network weights to reconstruct the target concept using diffusion loss~\cite{Ho2020DenoisingDP}. Custom Diffusion~\cite{kumari2022customdiffusion} works in a similar way but only fine-tunes a subset of the diffusion model layers, namely the cross-attention layers. LoRA~\cite{hu2022lora} is a low-rank matrix decomposition method that enables better parameter efficiency for fine-tuning methods, and was recently adapted to customization of text-to-image diffusion models~\cite{db_lora} (DB-LoRA). 
Recent works~\cite{sohn2023styledrop,ruiz2023hyperdreambooth,jia2023taming,Su2023IdentityEF,ye2023ip-adapter,Wang2023StyleAdapterAS,Shi2023InstantBoothPT} try to improve speed by training feed-forward networks to predict adaptation parameters from data, successfully amortize the time taken to create customize concepts.

\vspace{-1.2em}
\paragraph*{Multi-concept Customization.} 
Certain existing works have taken the task of customization one step further, aiming to inject multiple novel concepts into a model at the same time. Custom Diffusion~\cite{kumari2022customdiffusion} achieves this through a joint optimization loss for all concepts, while Break-a-scene~\cite{Avrahami2023BreakASceneEM} and SVDiff~\cite{han2023svdiff} introduces a masked cross-attention loss to learn individual concepts in images containing multiple concepts. However, such methods require access to ground truth data of all concepts training. In this work, we are interested in the task of \textbf{\textit{modular customization}}, where concepts are learned independently, and users can then mix and match individual concepts during inference for multi-concept image synthesis (Sec.~\ref{sec:mod-cus}). 

Prior works have provided implicit solutions to the problem of modular customization, but each existing method comes with its own set of trade-offs. TI~\cite{gal2022image, mokady2023null, voynov2023p+} implicitly addresses the task by representing each concept through a unique token embedding, enabling multi-concept customization by simply querying each token. However, TI tends to suffer from low subject fidelity, as token embeddings alone provide limited expressivity. Federated Averaging (FedAvg)~\cite{mcmahan2023communicationefficient} merges fine-tuned models by simply taking a weighted average between the weights of each model, although fast and expressive, naive combination tends to lead to loss of concept identity.
Custom Diffusion~\cite{kumari2022customdiffusion} supports merging of individually fine-tuned networks through solving a constrained customization problem. This method also struggles with expressivity, as only a small subset of the diffusion model weights are being updated. Concurrent work, Mix-of-Show (MoS)~\cite{gu2023mix} expands on this method by introducing gradient fusion, enabling merging of multiple separately fine-tuned models without placing restrictions on parameter expressivity. Though expressive, gradient fusion is computationally demanding, taking $\sim$15-20 minutes just to combine three custom concepts into a single model, which becomes intractably expensive when deployed at scale.
Table~\ref{tab:related_work} summarizes the key areas in which our approach differs from previous and concurrent works.

\section{Method}\label{sec:setup}
\begin{figure}[t]

\includegraphics[width=\linewidth]{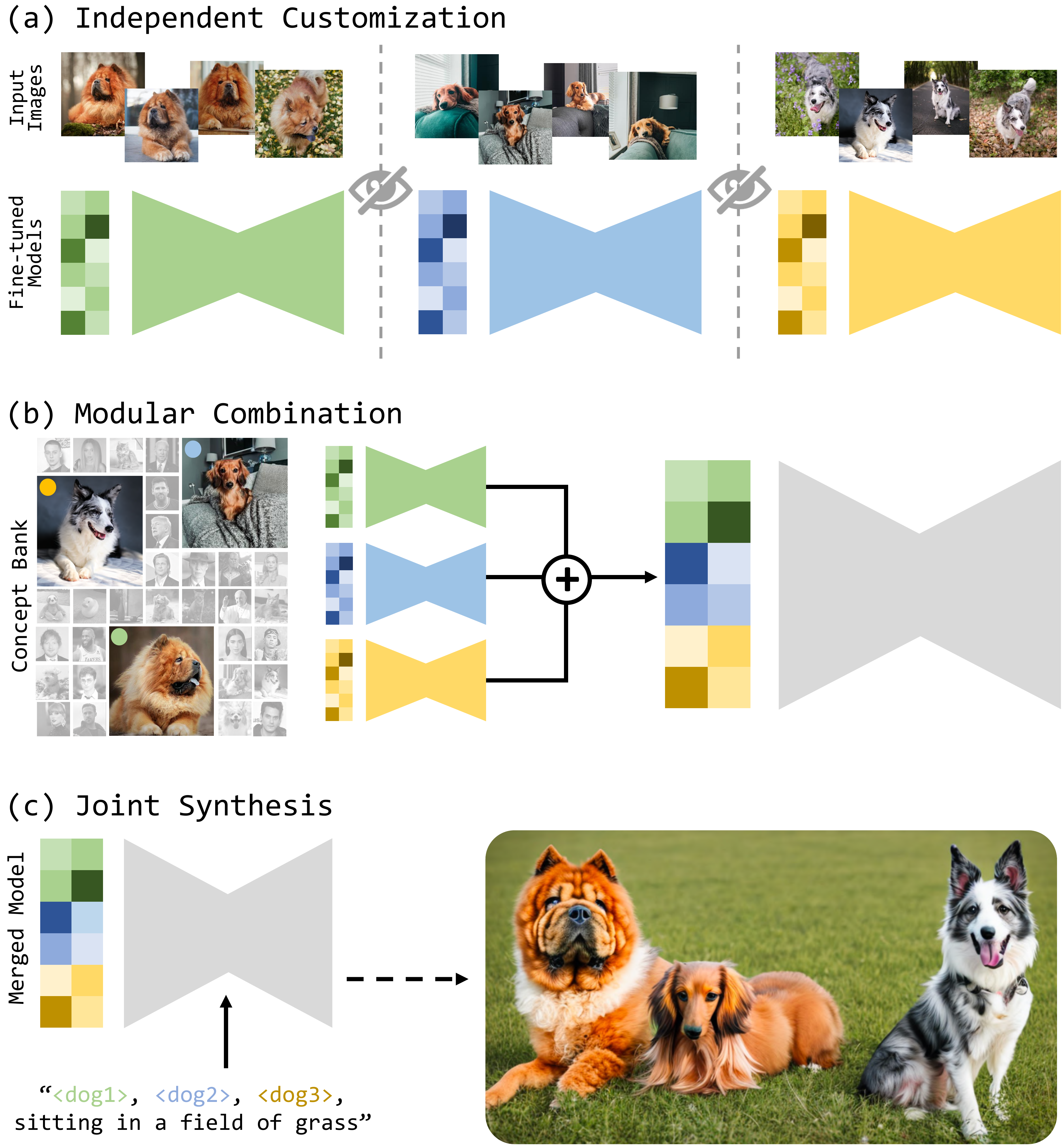}
\caption{The three stages of \textbf{Modular Customization}: (a) Independent Customization, (b) Modular Combination, and (c) Joint Synthesis. Note that during individual fine-tuning, all processes are private, meaning each user does not have access to ground truth data for other concepts. 
}
\label{fig:setup}
\vspace{-1.5em}
\end{figure}

In this section, we first introduce the problem setting of modular customization (Sec.~\ref{sec:mod-cus}). 
We then take a look at the simple solution of FedAvg~\cite{mcmahan2023communicationefficient}, and explore where and why this naive method fails to preserve identity (Sec.~\ref{sec:fedavg}). Motivated by the limitations of FedAvg, we discuss the conditions to ensure concept identity preservation (Sec.~\ref{sec:preserve}), and finally introduce our solution to modular customization -- \textbf{\textit{orthogonal adaption}} (Sec.~\ref{sec:orth-adapt} and Sec.~\ref{sec:basis}).

\subsection{Modular Customization}\label{sec:mod-cus}
In this paper, we are interested in customizing text-to-image diffusion models to generate multiple personal concepts in an efficient, scalable, and decentralized manner. 
\begin{figure*}[t!]
\includegraphics[width=\linewidth]{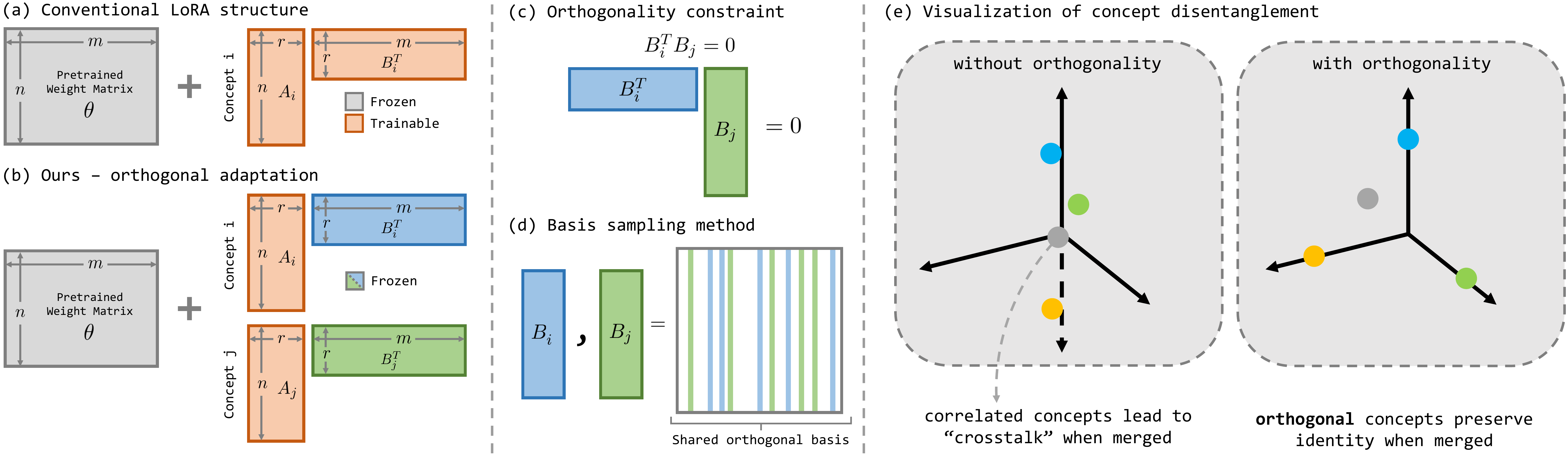}
\caption{\textbf{Overview of Orthogonal Adaptation.} (a) LoRA~\cite{hu2022lora} enables training of both low-rank decomposed matrices. (b) Orthogonal adaption constrains training only to $A$, leaving $B$ fixed. (c) For two separate concepts, $i$ and $j$, an orthogonality constraint is imposed between $B_i$ and $B_j$. (d) When concepts $i$ and $j$ are trained independently, approximate orthogonality between $B_i$ and $B_j$ can be achieved by sampling random columns from a shared orthogonal matrix. (e) Without the orthogonality constraint, correlated concepts suffer from ``crosstalk'' when merged; with the orthogonality constraint, orthogonal concepts preserve their identities after merging.}
\label{fig:method}
\vspace{-1.5em}
\end{figure*}
In addition to single-concept text-to-image customization, users are usually interested in seeing multiple concepts interacting together.
This calls for a text-to-image model that is customized to a set of concepts.
Being able to generate multiple personalized concepts in a single model, however, is challenging.
First, the number of sets containing all possible combinations of concepts is growing exponentially with respect to the number of  concepts -- an intractable number even for a relatively small number of concepts. 
As a result, it's important for personalized concepts to be merged with interactive speed.
Furthermore, users usually have limited compute at their end, which means any computation done on the users end should ideally be trivial.

These requirements motivate an efficient and scalable fine-tuning setting we call \textbf{\textit{modular customization}}, where individual fine-tuned models should act like independent modules, which can be combined with others in a plug-and-play manner without additional training. The setting of modular customization involves three stages: independent customization, modular combination and joint synthesis. Fig.~\ref{fig:setup} provides an illustration of this three stage process. 

With modular customization in mind, our goal is to design a fine-tuning scheme, such that individually fine-tuned models can be trivially combined (e.g. summation) with any other fine-tuned model to enable multi-concept generation.

\subsection{Federated Averaging}\label{sec:fedavg}
Perhaps the most straight-forward technique for achieving modular customization is to take a weighted average of each individually fine-tuned model. This technique is often referred to as FedAvg~\cite{mcmahan2023communicationefficient}. Given a set of learned weight residuals $\Delta\theta_i$ optimized on concept $i$, the resulting merged model is simply given by
\begin{equation}
    \theta_{\text{merged}} = \theta + \sum_i \lambda_i\Delta\theta_i,
    \label{eq:fedavg}
\end{equation}
where $\theta$ represents the pre-trained parameters of the model used for fine-tuning, and $\lambda_i$ is a scalar representing the relative strength of each concept. While FedAvg is fast and places no constraints on the expressivity of each individually fine-tuned model, naively averaging these weights can lead to loss of subject fidelity due to interference between the learned weight residuals. This effect is especially severe when training multiple semantically similar concepts (e.g., human identities), as learned weight residuals tend to be very similar. 
We coin this undesirable phenomenon ``crosstalk''. Fig.~\ref{fig:multi-concept} and Fig.~\ref{fig:ablations}(a) provide visualizations of the effect of crosstalk, as FedAvg causes multi-concept generations to exhibit loss of identity.
Our approach is inspired by FedAvg.
We adopt its computational efficiency but modify the fine-tuning process to ensure  minimal interference between learned weight residuals between different concepts.
We want to enable instant, multi-concept customization from individually trained models without sacrificing subject fidelity.

\subsection{Preserving Concept Identity}\label{sec:preserve}
With the goal of addressing the limitations of FedAvg, we first examine where this method fails. For simplicity, consider the case of merging two concepts $i$ and $j$. After fine-tuning on each individual task, we receive a set of learned weight residuals $\Delta\theta_i$ and $\Delta\theta_j$. The output of a particular linear layer in the fine-tuned network is
\begin{equation}
    O_i(X_i) = (\theta + \Delta\theta_i)X_i,
\end{equation}
where $X_i$ represents a particular input to the layer corresponding to the training data of concept $i$. When merging the two concepts using FedAvg with $\lambda = 1$, the resulting merged model produces
\begin{equation}
    \hat{O}_i(X_i) = (\theta + \Delta\theta_i + \Delta\theta_j)X_i.
\end{equation}
The goal of concept preservation is to have $\hat{O}_i(X_i) = O_i(X_i)$. Note that, without enforcing specific constraints, it is likely that $\Delta\theta_jX_i\neq0$ and, subsequently, $\hat{O}_i \neq O_i$.

It follows that the mapping of data for concept $i$ is preserved when $\Delta\theta_j X_i = 0$ for $j \neq i$. By symmetry, the mapping of data for concept $j$ is preserved given $\Delta\theta_i X_j = 0$ for $i \neq j$. Intuitively, $||\Delta\theta_j X_i||$ measures the amount of crosstalk between the customized weights of concepts $i$ and $j$. We would like to keep this value low to ensure subject identity is preserved even after merging. However, note that given enough data for training a certain concept $i$, $X_i$ is likely to have full column rank. This makes the orthogonality condition impossible to satisfy. Instead, we propose a relaxation to this condition, choosing to minimize the crosstalk term for some projection of $X_i$ onto a subspace $S_i$. This projection yields $S_iS^T_iX_i$, and our relaxed objective hopes to achieve $\hat{O}_i(S_iS^T_iX_i) = O_i(S_iS^T_iX_i)$.

\subsection{Orthogonal Adaptation}\label{sec:orth-adapt}
Motivated by the relaxed objective above, we propose \textit{\textbf{orthogonal adaptation}}. Similar to low-rank adaptation (LoRA), we represent learned weight residuals through a low-rank decomposition of the form
\begin{equation}
    \Delta \theta_i = A_iB_i^T, \theta_i\in\Re^{n\times m} A_i\in\Re^{n\times r}, B_i\in\Re^{m\times r},
\end{equation}
where the rank $r <\!\!< \min(n,m)$. However, contrary to conventionally fine-tuning with LoRA, we keep $B_i$ constant, and only optimize $A_i$. 

Consider a matrix $\bar{B}_j$, where its columns span the orthogonal complement of the column space of $B_j$.
We show that by selecting $S_i = \bar{B}_j$, we achieve the conditions for achieving the projected preservation objective. This can be seen from the fact that,
\begin{align}
    \hat{O}_i(S_iS^T_iX_i) &= O_i(S_iS^T_iX_i) + \Delta\theta_jS_iS^T_iX_i \\
    &= O_i(S_iS^T_iX_i) + A_j\cancelto{0}{B^T_jS_i}S^T_iX_i  \\
    &= O_i(S_iS^T_iX_i).
\end{align}
Since $r <\!\!< m$, the orthogonal complement of $B_j$ covers most of $\Re^m$. It follows that $\bar{B}_j\bar{B}^T_jX_i \approx X_i$, making $\bar{B}_j$ a reasonable candidate for $S_i$.

At the same time, since we expect the learned residuals for a concept to have meaningful interactions with their data, we would also like to ensure $||\Delta\theta_iX_i||$ is non-trivial. By approximating $X_i$ with its projection onto $\bar{B}_j$, our objective changes to ensuring $||A_iB^T_i\bar{B}_j\bar{B}^T_jX_i||$ is non-trivial. Examining this term gives us the additional constraint that $B^T_i\bar{B}_j \neq 0$, meaning the columns of $B_i$ should live in the orthogonal complement of the columns space of $B_j$. Therefore, to ensure meaningful fine-tuning results, we should also enforce orthogonality between the learned residuals, i.e. $B_i^TB_j = 0$.

Fig.~\ref{fig:method} provides an overview of our orthogonal adaption method. Intuitively, as illustrated in Fig.~\ref{fig:method}(e), our method disentangles custom concepts into orthogonal directions, ensuring that there is no crosstalk between concepts. 
As a result, our merged model can better preserve the identity of each concept.
\begin{figure}[t!]
\centering
\includegraphics[width=\linewidth]{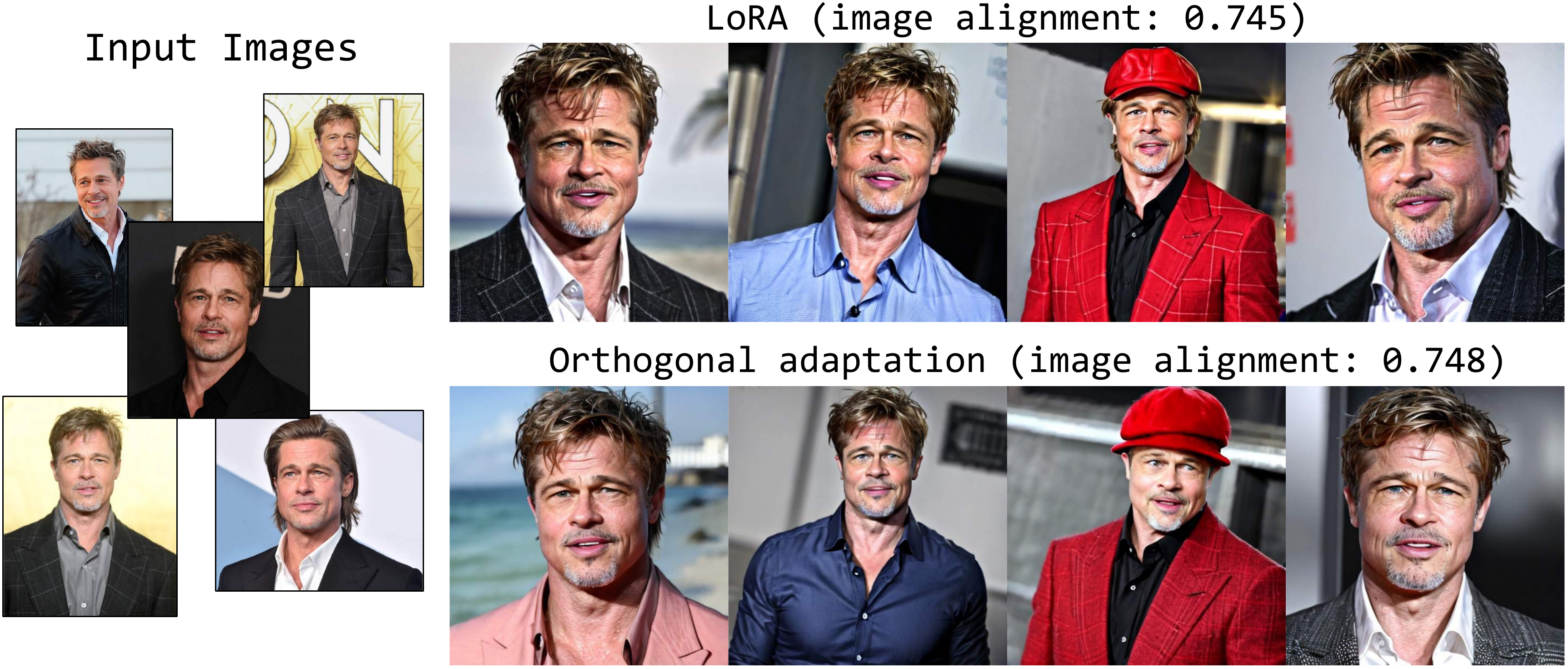}
\vspace{-2em}
\caption{\textbf{Over-parameterization of text-to-image models.} 
Despite the added constraint on the trained weight residuals, due to the over-paramterized nature of large text-to-image diffusion models, our method is able to achieve single-concept customization results with comparable fidelity to the unconstrained setting.
}
\label{fig:constrained-optimization}
\vspace{-1em}
\end{figure}
\vspace{-1em}
\paragraph{Expressivity of orthogonal adaption.} Expressivity of our method arises as a natural concern as we are optimizing significantly fewer parameters by freezing $B_i$. Fortunately, text-to-image diffusion models are often over-parameterized, with millions/billion of parameters. Prior works have shown that even fine-tuning a subspace of such parameters can be expressive enough to capture a novel concept. We also show this result empirically in Fig.~\ref{fig:constrained-optimization}, where our method leads to results with similar fidelity, even without the need to optimize $B_i$ during training.

\subsection{Designing Orthogonal Matrices $B_i$'s}\label{sec:basis}
A key challenge of the method described in previous sections is to generate a set of basis matrices $B_i$ that are orthogonal to each other.
Note that this is very difficult especially because when choosing $B_i$, the user is not aware of what basis the other users chose to optimize for the concepts to be combined in the future.
Strictly enforcing such orthogonality might be infeasible without prior knowledge of other tasks.
We instead propose a relaxation to the constraint, introducing a simple and effective method to achieve approximate orthogonality.
\vspace{-1em}
\paragraph{Randomized orthogonal basis.} One method for enforcing approximate orthogonality is to determine a shared orthogonal basis.
For some linear weight $\theta \in \Re^{m\times n}$, we first generate a large orthogonal basis $O\in \Re^{n\times n}$. This orthogonal basis is shared between all users. During training of concept $i$, $B_i$ is formed from taking a random subset of $k$ columns from $O$. Given $k <\!\!< n$, the probability of two randomly chosen $B_i$'s to share the same columns is kept low.
\vspace{-2em}
\paragraph{Randomized Gaussian.} Another approach is to choose random matrix elements.
Specifically, we sample each entry of $B_i$ from a zero-mean Gaussian with standard deviation $\sigma$: $B_i[k] \sim \nN(0, \sigma^2 I)$.
When the dimensionality of $B_i$ is high, this simple strategy creates matrices that are orthogonal in expectation: $\Ee\left[B_i^TB_j \right] = 0$ (see supplement for discussion). Naturally, this method does not require knowledge of a shared basis to sample from. In practice, however, we found randomized Gaussians lead to higher levels of crosstalk in our setting, i.e., $||B_i^TB_j||$ tends to be larger than for the randomized orthogonal basis.


\begin{figure*}[t!]
    \centering
\includegraphics[width=\linewidth]{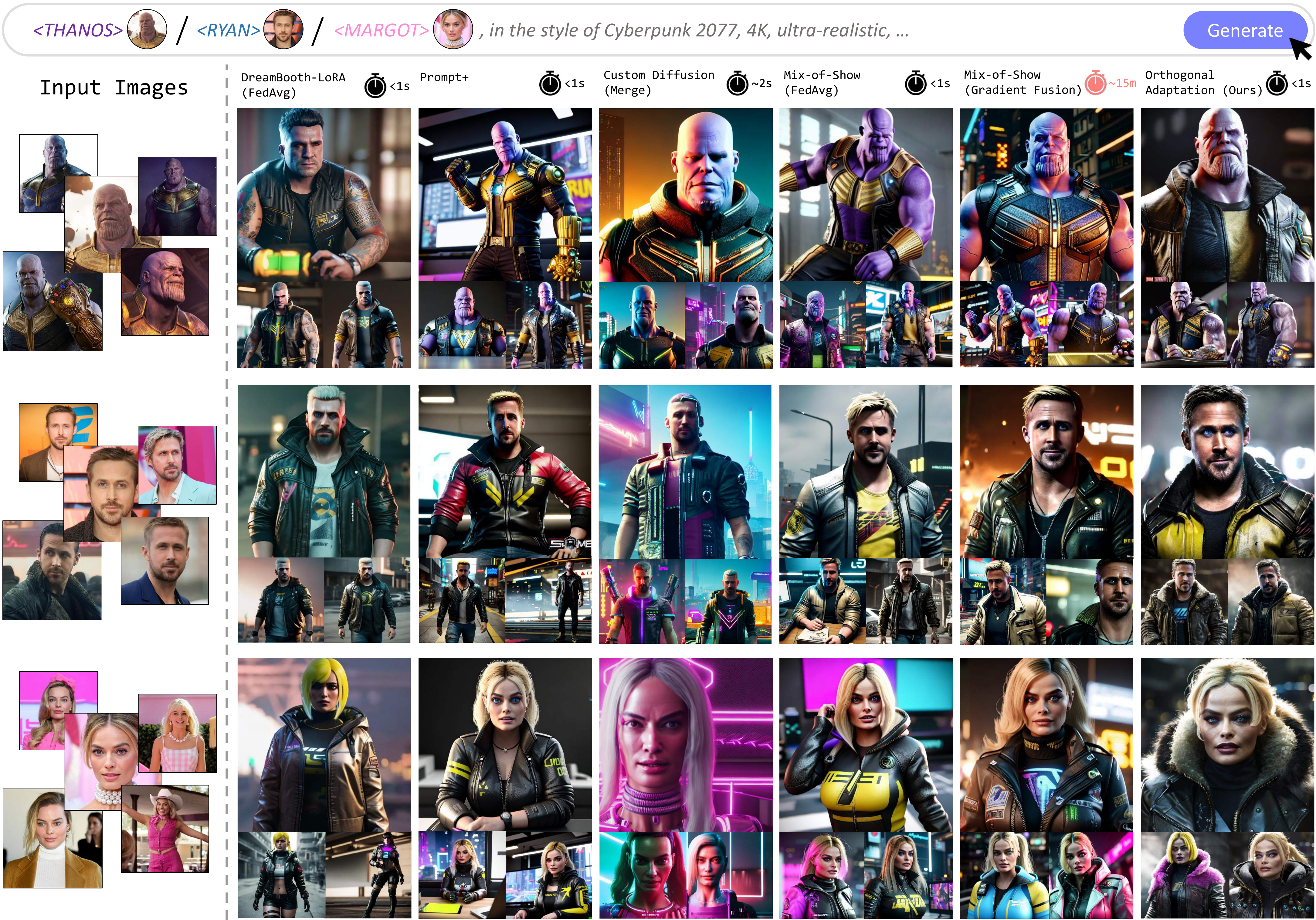}
    \vspace{-1.8em}
    \caption{\textbf{Identity preservation in single-concept generations from a merged model.} We demonstrate our method's ability to maintain identity consistency across different single-concept generations. Each column showcases images from the same merged model, representing three distinct concept identities. Our approach showcases better identity alignment with the corresponding input images, offering a significant improvement over comparable merging methods. Additionally, our method's performance parallels that of Mix-of-Show (Gradient Fusion) but with the advantage of near-instantaneous merging, in contrast to the approximately 15-minute merging time required.
    }
    \vspace{-1.6em}
    \label{fig:identity}
\end{figure*}
\section{Experiments}In this section, we show the results of our method applied to the task of modular customization. Qualitative and quantitative results indicate that our method outperforms relevant baselines~\cite{gu2023mix,db_lora,kumari2022customdiffusion} at similar speeds, and quality on par with state-of-the-art baselines that require significantly higher processing times~\cite{gu2023mix}.
\vspace{-1em}
\paragraph{Datasets.}We perform evaluations on a custom dataset of 12 concept identities, each containing 16 unique images of the target concept in different contexts. 
\vspace{-1em}
\paragraph{Implementation details.}We perform fine-tuning on the Stable Diffusion~\cite{Rombach2021HighResolutionIS} model, specifically the ChilloutMix checkpoint for its ability to handle high-fidelity human face generation. For single-concept fine-tuning, we apply orthogonal adaptation to all linear layers in the Stable Diffusion architecture. Following prior work~\cite{voynov2023p+,gu2023mix}, we also apply a layer-wise text embedding and represent each fine-tuned concept as two separate text tokens. We fine-tune the text embeddings with a learning rate of $1e-3$, the diffusion model parameters with a learning rate of $1e-5$ and set $r = 20$ for all experiments. Single-concept fine-tuning takes $\sim$10-15 minutes on two A6000 GPUs. For our method, we enforce the orthogonality constraint using the randomized orthogonal basis method for all experiments. Methods using FedAvg (including orthogonal adaption) were merged using $\lambda = 0.6$.
\begin{figure*}[t!]

    \centering
\includegraphics[width=\linewidth]{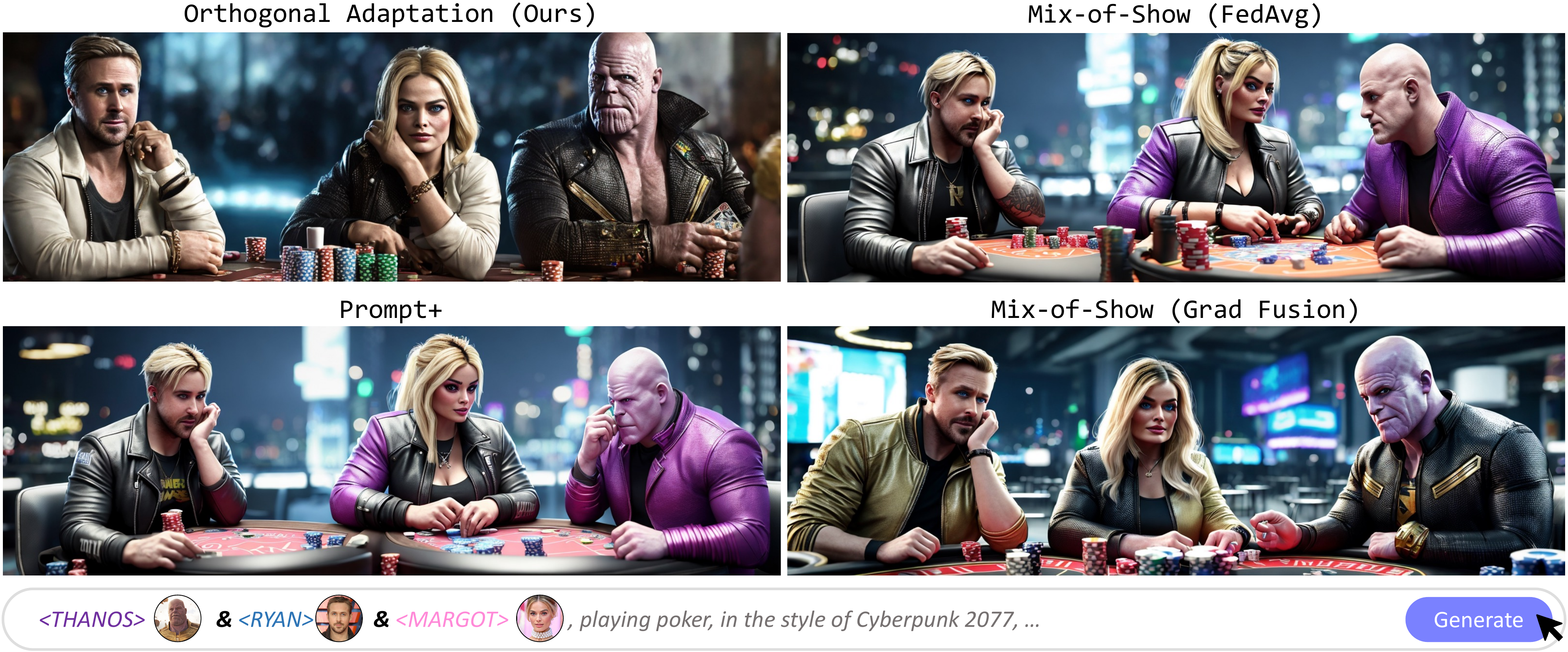}
\vspace{-1.8em}
    \caption{\textbf{Multi-concept results.} Examples of multi-concept generations, synthesized using sampling techniques from concurrent work~\cite{gu2023mix}. While Mix-of-Show (FedAvg) maintains high-level features, it struggles with crosstalk, manifesting overly smooth facial features. Mix-of-Show (Gradient Fusion) exhibits good identity alignment, albeit with a computationally intensive merging process. $\mathcal{P}+$ manages to preserve identity after merging, but struggles to capture identity with high-fidelity due to limited parameter expressivity. Our method stands out by achieving high identity alignment with a significantly faster merging procedure. }
    \label{fig:multi-concept}
\end{figure*}

\begin{table*}[t!]
    \vspace{-0.5em}
  \centering
  \begin{tabular}{llccccccccc}
  \toprule
    \multirow{2}{*}[-0.1cm]{Method}& \multirow{2}{*}[-0.1cm]{\shortstack{Merge \\ Time}} & \multicolumn{3}{c}{Text Alignment $\uparrow$} & \multicolumn{3}{c}{Image Alignment $\uparrow$} & \multicolumn{3}{c}{Identity Alignment $\uparrow$} \\
    \cmidrule{3-11}
    & & Single  & Merged  & $\Delta$ & Single  & Merged  & $\Delta$ & Single  & Merged  & $\Delta$\\
    \midrule
    P+~\cite{voynov2023p+} & $<$1 s & .643\RA & .643 & --- & .683\RA & .683 & --- & .515\RA & .515 & --- \\
    Custom Diffusion~\cite{kumari2022customdiffusion} & $\sim$2 s & .668\RA & .673 & +.005 & .648\RA & .623 & -.025 & .504\RA & .408 & -.096 \\
    DB-LoRA (FedAvg)~\cite{db_lora} & $<$1 s & .613\RA & .682& +.069 & .744\RA & .531 & -.213 & .683\RA  & .098 & -.585 \\
    MoS (FedAvg)~\cite{gu2023mix}& $<$1 s & .625\RA & .621 & -.004 & .745\RA & .735 & -.010 & .728\RA & .706 & -.022 \\
    MoS (Grad Fusion)~\cite{gu2023mix}& \color{red}{$\sim$15 m} & .625\RA & .631 & +.006 & .745\RA& .729 & -.016 & .728\RA & .717 & -.011 \\
    \cdashlinelr{1-11}
    Ours & $<$1 s & .624\RA & .644 & -.010 & .748\RA & \textbf{.741} & -.007 & .740\RA & \textbf{.745} & +.005\\
    \bottomrule
  \end{tabular}
  \vspace{-0.8em}
    \caption{\textbf{Quantitative results.} We provide detailed qualitative comparisons for each method, evaluated both before and after the merging process. Prior to merging, our method demonstrates comparable performance in all identity-related metrics, highlighting its expressivity even with the orthogonality constraint. Post-merging, our method achieves the highest scores in image and identity alignment. Our method is also capable of maintaining text alignment scores comparable to other high-fidelity methods such as P+ and MoS.}
    \vspace{-1em}
  \label{tab:speed_eval}
\end{table*}

\paragraph{Baselines.} We compare our method against state-of-the-art baselines on the task of modular customization, namely: DreamBooth-LoRA~\cite{db_lora}, $\mathcal{P}+$~\cite{voynov2023p+}, Custom Diffusion~\cite{kumari2022customdiffusion}, and Mix-of-Show~\cite{gu2023mix}. Fine-tuned models are merged differently depending on the method. DreamBooth-LoRA is merged using FedAvg, Custom Diffusion is merged using their proposed optimization-based merging method, and Mix-of-Show is merged using gradient fusion as outlined in their work. Since $\mathcal{P}+$ does not perform fine-tuning on the weights of the network, merging is done simply by querying each concept's token embedding. For completeness, we also compare against Mix-of-Show merged using FedAvg, serving as an efficient alternative to the computationally demanding gradient fusion method.
\vspace{-1em}
\paragraph{Experimental setup and metrics.}First, we fine-tune each concept individually, without access to data for any other concept. Each fine-tuned model is then combined with two other concepts at random using their corresponding method for merging. Following prior work, we evaluate our method on \textit{image alignment}, which measures the similarity of image features between generated images and the input reference image by measuring their similarity in the CLIP image feature space~\cite{gal2022image}. Similarly, we evaluate our method using \textit{text alignment}, ensuring the output generations still adhere to the input text-prompts by measuring the text-image similarity also using CLIP~\cite{hessel2022clipscore}. However, to further illustrate the identity preserving capabilities of our method, we also evaluate our method using the ArcFace~\cite{Deng_2022} model. Using the ArcFace model, we measure the rate at which the target human identity is detected in a set of generated images, we refer to this metric as \textit{identity alignment}. 

\subsection{Qualitative Comparisons}
\paragraph{Merged single-concept results.} We illustrate the identity preserving effect of our method by comparing single-concept generations of different identities from the same merged model. As mentioned above, each concept is fine-tuned individually and merged together during inference. Fig~\ref{fig:identity} shows generations for three separate concept identities, each column contains images sampled from the same model. Our method achieves better identity alignment with the input images in the merged model compared to methods with comparable merging times. We also achieve similar results to Mix-of-Show (Gradient Fusion), which requires $\sim$15 minutes to merge three concepts, while our method enables near instant merging.

\vspace{-1em}
\paragraph{Merged multi-concept results.} We also show generated images containing all three identities in the merged model. Leveraging multi-concept sampling techniques from concurrent work~\cite{gu2023mix}, we show examples of multi-concept generations in Fig.~\ref{fig:multi-concept}. Once again, multi-concept models trained using our method generate images with better identity alignment than competing baselines. Due to the poor performance of DB-LoRA~\cite{db_lora} and Custom Diffusion~\cite{kumari2022customdiffusion} for single-concept generations, we omit results for these methods on multi-concept generation due to space constraints.

$\mathcal{P}+$~\cite{hertz2022prompt} suffers from low concept fidelity due to limited expressivity in their training regime. Although Mix-of-Show~\cite{gu2023mix} (FedAvg) preserves certain high-level features through the layer-wise text-embedding, it still suffers from crosstalk due to unconstrained training of weight residuals. Mix-of-Show (Gradient Fusion) shows impressive identity alignment, however, this is only enabled by a computationally demanding merging procedure. Our method achieves high identity alignment while keeping the merging process at near instant rates.

\subsection{Quantitative Results}
We present quantitative comparisons in Table.~\ref{tab:speed_eval}. Specifically, we show all three evaluation metrics applied to each method before and after merging. Our method achieves comparable results in all concept alignment metrics before merging, illustrating the expressivity of our method despite the orthogonality constraint. After merging, our method achieves the highest image and identity alignment scores across all methods, while maintaining comparable text alignment scores with other high-fidelity methods such as Mix-of-Show and $\mathcal{P}+$. This illustrates that our method is able to achieve high identity preservation without sacrificing the ability to generalize for different contexts. 

Note that although Custom Diffusion~\cite{kumari2022customdiffusion} and DB-LoRA~\cite{db_lora} achieves higher text alignment, this is at the cost of significantly lower concept alignment scores than that of competing methods.

\section{Ablations}
\paragraph{Effect of orthogonality.} In Fig.~\ref{fig:ablations}(a), we present generated images from a model created from merging two separate fine-tuned models (concepts $i$ and $j$). To illustrate the effect of orthogonality on identity preservation, we manipulate the degree of orthogonality between $B_i$ and $B_j$. On the left, we have the worst case scenario, where $B_i = B_j$. On the right, we show results where perfect orthogonality is achieved, i.e. $B_i^TB_j = 0$. In between, we construct $B_i$ and $B_j$ from a shared orthogonal matrix, but choose half of their columns to be overlapping. Results in Fig.~\ref{fig:ablations}(a) show that orthogonality contributes significantly to identity preservation even in the extreme case of merging 2 concepts.

\vspace{-1em}
\paragraph{Number of merged concepts}
Fig.~\ref{fig:ablations}(b) shows results generated from models with a range of concepts merged together. With orthogonality, our model is capable of merging a high number of concepts with minimal identity loss. In contrast, without orthogonality, concept fidelity quickly degrades, even with relatively low number of concepts being combined. Running our model without orthogonality is equivalent to Mix-of-Show~\cite{gu2023mix} merged using FedAvg~\cite{mcmahan2023communicationefficient}.


\section{Discussion}

\paragraph{Limitations.}
Despite showcasing the ability to encode several custom concepts into the same text-to-image model, generating images with complex compositions/interactions between multiple custom concepts remains challenging. As concepts, such as human identities, have the tendency to either be entangled, or even completely ignored. Existing works~\cite{gu2023mix, bar2023multidiffusion} have developed certain strategies for remedying this effect, but such methods are still prone to the aforementioned failure cases. Another limitation of orthogonal adaption is that it directly modifies the fine-tuning process. Therefore, existing fine-tuned networks (e.g. LoRAs~\cite{db_lora}) can not be adapted post-hoc to ensure orthogonality.
\begin{figure}[t!]
\vspace{-0.8em}
\centering
\includegraphics[width=\linewidth]{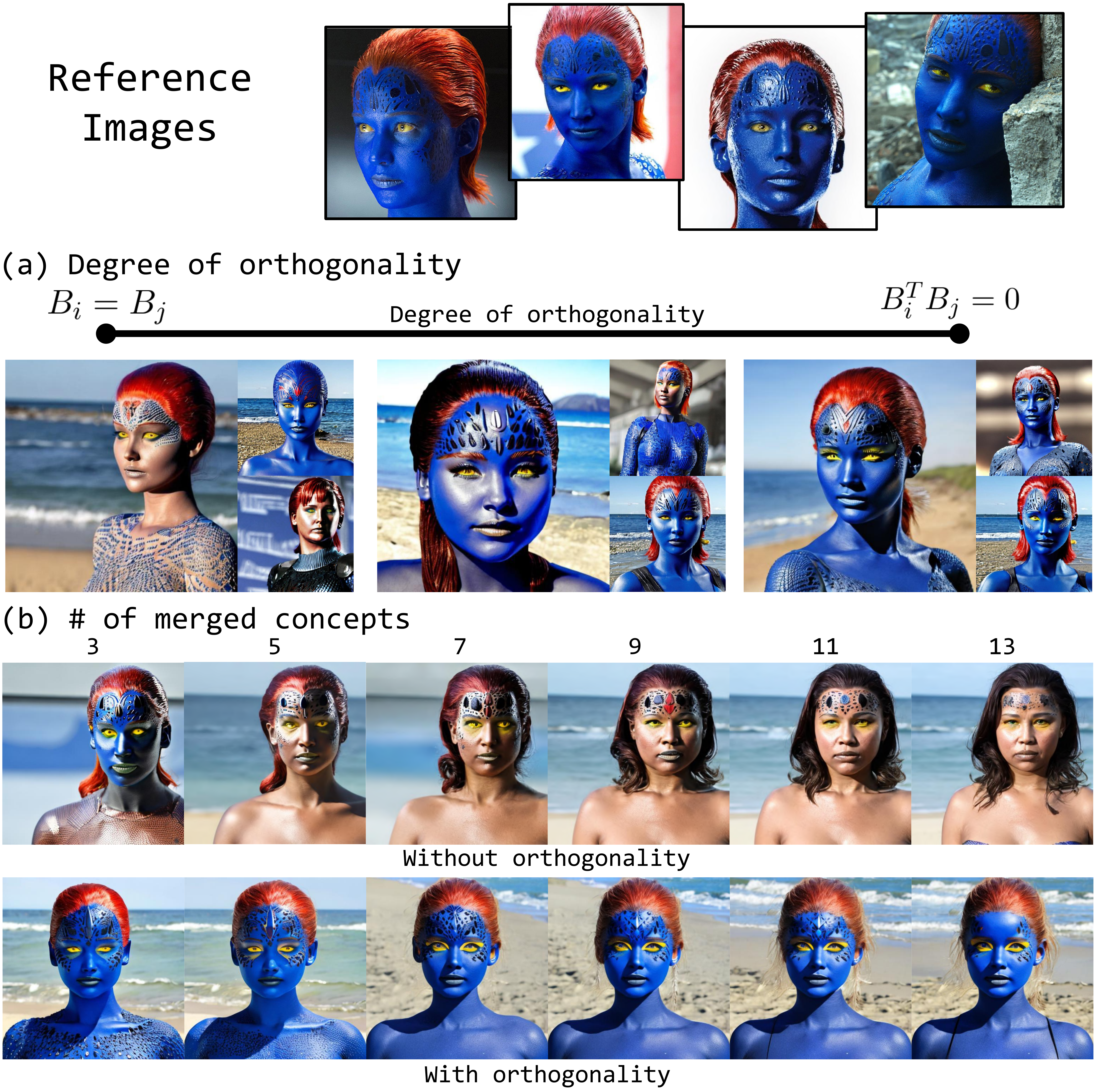}
\vspace{-1.8em}
\caption{\textbf{Ablation studies.} 
(a) Images generated from a model formed by merging two separately fine-tuned models (concepts i and j), focusing on the role of orthogonality in preserving identity. (b) Image generations from models that with a varying number of merged concepts. Without orthogonality, concept identity is lost even when merging a small number of concepts.
}
\label{fig:ablations}
\vspace{-1.2em}
\end{figure}

\vspace{-1.2em}
\paragraph{Ethics Considerations.}
Generative AI could be misused for generating edited imagery of real people with the intent of spreading disinformation. Such misuse of image synthesis techniques poses a societal threat, and we do not condone using our
work for such purposes. We also recognize a potential biases in the foundation model we built upon. 
\vspace{-1.2em}
\paragraph{Conclusions.} By disentangling customization concepts into orthogonal directions, orthogonal adaptation streamlines the process of integrating multiple independently fine-tuned concepts into a single model instantly and with trivial compute, while also ensuring preservation of each concept. Our work makes a significant step towards modular customization, where multi-concept customization can be achieved with individual, privately fine-tuned models.

\section{Acknowledgements}
We thank Youjin Song for developing the hugging-face demo, as well as Sara Fridovich-Keil and Kamyar Salahi for fruitful discussions and pointers for evaluation metrics. Po is supported by the Stanford Graduate Fellowship. This project was in part supported by Samsung and Stanford HAI.
{
    \small
    \bibliographystyle{ieeenat_fullname}
    \bibliography{references}
}

\clearpage
\setcounter{page}{1}
\maketitlesupplementary

\section{Gaussian random orthogonal matrices}
\begin{theorem}
    Let $\vv\in \Re^d$ and $\uu\in \Re^d$ be two random vectors.
    Let $\vv_i\sim \nN(0, \sigma^2 I)$ and $\uu_i\sim \nN(0, \sigma^2 I)$ for all $i\in [1, d]$ independently, then $\Ee\left[ \vv^T\uu\right] = 0$.
\end{theorem}
\begin{proof}
    \begin{align*}
        \Ee\left[\vv^T\uu\right] 
        &= \Ee\left[\sum_{i=1}^d \vv_i\uu_i \right]  & \\
        &= \sum_{i=1}^d \Ee\left[ \vv_i\uu_i \right]  &\text{(Linearity of expectation)}\\
        &= \sum_{i=1}^d \Ee[\vv_i]\Ee[\uu_i] 
        &\text{(Independent)}\\
        &= \sum_{i=1}^d 0\cdot 0 = 0.
    \end{align*}
\end{proof}

\begin{corollary}
    Let $\AA\in \Re^{n\times m}$ and $\BB\in \Re^{n\times m}$.
    All entries of these matrices are independently sampled from $\nN(0, \sigma^2I)$.
    Then $\Ee[\AA^T\BB] = \mathbf{0} \in \Re^{m\times m}$.
\end{corollary}
\begin{proof}
    \begin{align*}
        \Ee[\AA^T\BB]_{ij} = \Ee[\AA_i^T\BB_j] = 0.
    \end{align*}
\end{proof}

\section{Implementation details}
\paragraph{Dataset.} We chose to evaluate our method on human datasets due to the robustness of face recognition algorithms for evaluation purposes. While prior works ~\cite{kumari2022customdiffusion, ruiz2023dreambooth, gu2023mix, han2023svdiff} have employed CLIP-based metrics as a method of evaluating identity alignment, we found that CLIP features are often poor at identifying fine details in a custom concept. In Fig.~\ref{fig:crosstalk_supp}, we illustrate that our method works for non-human objects too.
\begin{figure}[t!]
\centering
\includegraphics[width=\linewidth]{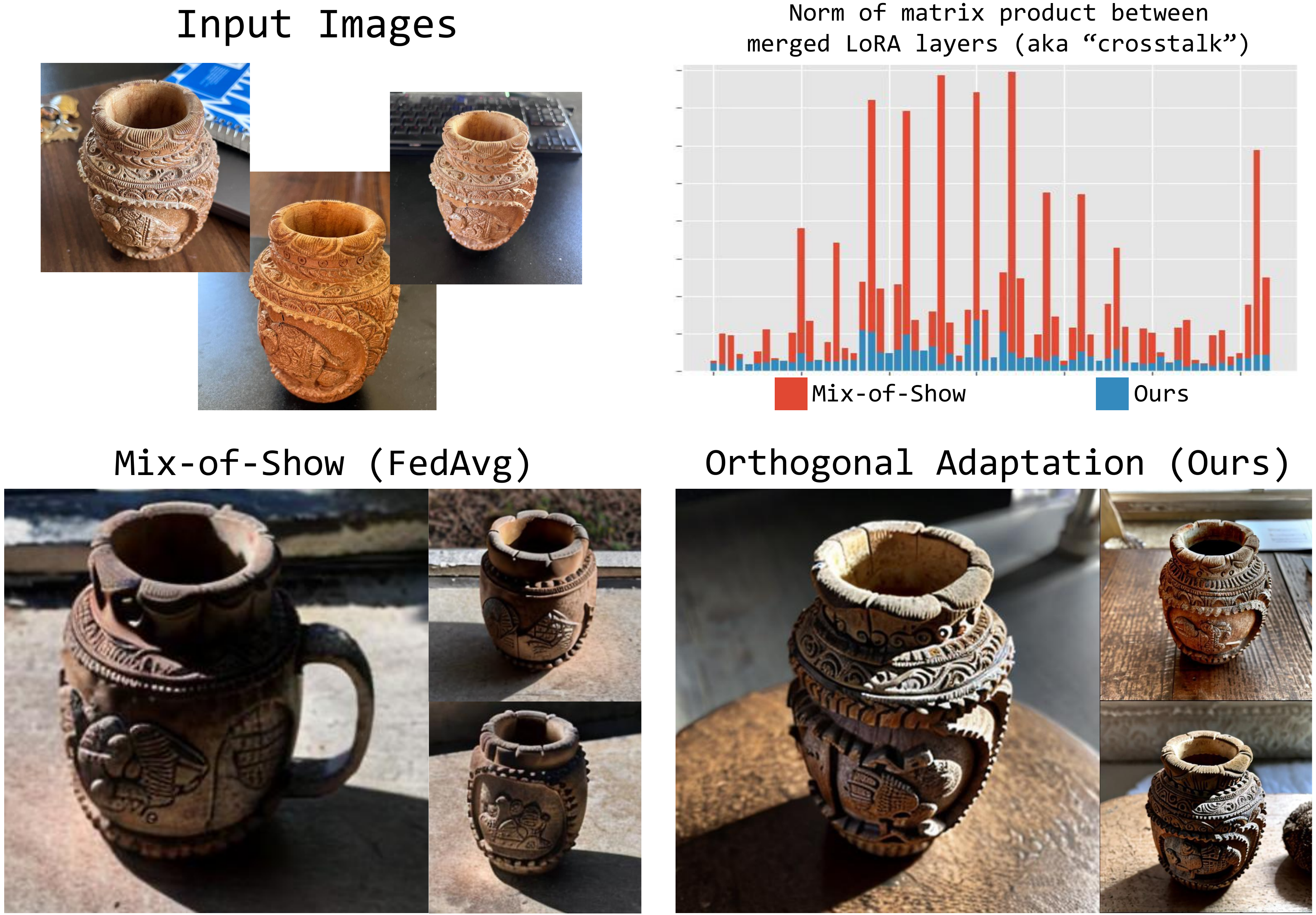}
\caption{\textbf{Identity loss due to crosstalk.} We illustrate the effects of crosstalk by examining the effects of interfering signals between independently trained LoRAs. Measuring crosstalk through the norm of the product between two LoRA weights, our method results in lower crosstalk between independently trained LoRAs. Combined via the same method, our training regime leads to less crosstalk and therefore better identity preservation after merging.}
\label{fig:crosstalk_supp}
\vspace{-1.5em}
\end{figure}
\paragraph{Evaluation details.} We introduce the \textit{identity alignment} metric for measuring the ability of our method (and competing baselines) in capturing the target human identity in resulting generations. We use the ArcFace~\cite{Serengil2020LightFaceAH} facial recognition algorithm and consider a detection to be recorded when the ArcFace distance between two detected faces falls below 0.680~\cite{Serengil2020LightFaceAH}. We choose to use detection probability as a metric rather than the raw distance metric as we found the distance metric to favor over-fitted models. Past the detection threshold, the distance metric directly measures the similarity between two faces, which is not ideal for use-cases such as re-stylization and accessorization.
\paragraph{Orthogonal adaptation details.} In our method, we enforce the orthogonality constraint through the LoRA down projection matrix $B$. This formulation ensures orthogonality in the row-space of the resulting LoRA matrices. In theory, we can also achieve orthogonality between trained weight residuals in the column-space, in which case the orthogonality constraint would have to be enforced on the up-projection matrix $A$ instead. We choose to enforce orthogonality in the row-space since the weight residuals interact with the layer inputs through their rows. The concept preservation formulation presented in Sec.~\ref{sec:setup} is also reliant on row-space orthogonality.
\begin{figure*}[t!]
\centering
\includegraphics[width=0.8\linewidth]{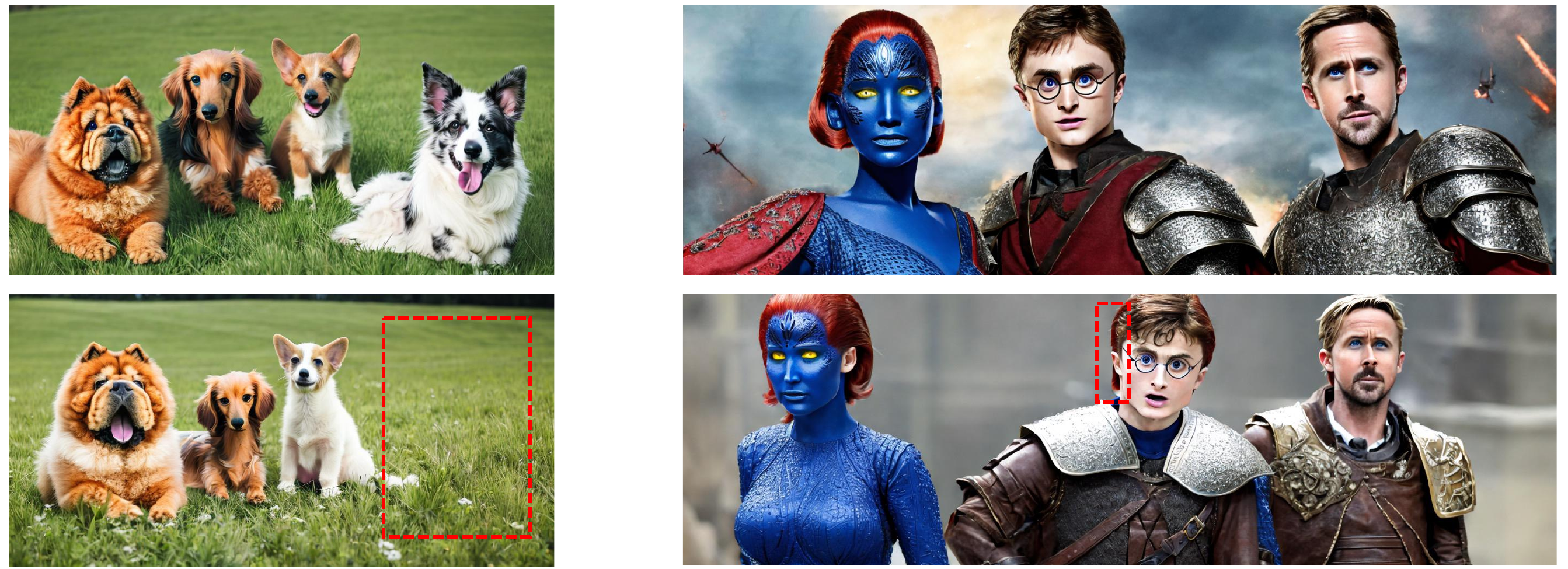}
\caption{\textbf{Multi-concept failure cases.} Multi-concept generation remains as an open challenge. Despite employing techniques such as regionally controllable sampling from prior work~\cite{gu2023mix}, this method can still suffer from failure cases such as: (left) ignoring concepts, and (right) leakage of concept attributes to neighboring identities.}
\label{fig:failure}
\vspace{-1.5em}
\end{figure*}
In our results, we chose to use the random orthogonal basis method for enforcing orthogonality in all our results. Although the Gaussian random method results in orthogonality on expectation, the orthogonal basis method led to lower crosstalk emperically. The orthogonal basis method requires a shared orthogonal matrix to sample from. In practice, using Stable Diffusion v1.5, there are only four unique input dimensions for all layers in the diffusion model (320, 640, 768, 1280). Therefore, we only have to store four unique square matrices from which all sampled $B_i$'s can then be sampled from. These four orthogonal matrices can be downloaded along with the base model, but they can also be generated on the fly with a fix seed to ensure they are shared among all users.

\paragraph{FedAvg merging coefficient.} Existing work considers FedAvg merging with affine coefficients. However, with a larger number of concepts, affinely combining each LoRA will lead to dilution of signal from individual LoRAs. It is also a common practice to scale individual LoRA weights post-hoc~\cite{db_lora} for direct control over the signal strength from the fine-tuning process. We combine this scaling factor along with the FedAvg merging factor to obtain a single scale factor $\lambda_i$ as shown in Eq.~\ref{eq:fedavg}. We consider merging coefficients as a hyper-parameter that can be tuned based on user preferences.

\section{Additional results}
\paragraph{Illustration of crosstalk.} Fig.~\ref{fig:crosstalk_supp} illustrates the importance of minimizing crosstalk for identity preservation when merging LoRA weights into a single model. We measure crosstalk formally using the norm of the matrix product between individually trained LoRA weight residuals. Upper right of Fig.~\ref{fig:crosstalk_supp} shwos a direct comparison of the layer-wise normalized matrix product norms between two LoRAs trained with and without orthogonality constraints. Our method leads to a much lower levels of crosstalk, which translates to better identity preservation as observed from the resulting generations.
\vspace{-1.2em}
\paragraph{Extended baseline comparisons.} In Fig.~\ref{fig:identity_supp} We show an extended version of Fig.~\ref{fig:identity} with generated images of each identity for each method before they are merged. These results aim to show that our method is capable of retaining identity alignment with the target concept before and after merging, while achieving merging of individual LoRAs instantly without any further fine-tuning or optimization stages.

\paragraph{Over-fitting.} Since we are fine-tuning our network over a small custom dataset and we initialize our custom tokens with a user-defined class label, it may be susceptible to over-fitting. Prior works such as DreamBooth~\cite{ruiz2023dreambooth} and Custom Diffusion~\cite{kumari2022customdiffusion} alleviate this effect by adding a class preservation loss that ensures generating images from the class token still produces diverse results. In our method, we do not employ an explicit loss to prevent over-fitting, however, we found that our fine-tuned models still preserve the ability to generate diverse images for the trained class label as shown in Fig.~\ref{fig:class}

\section{Limitations and future work}
Our method takes an important step towards achieving modular customization. However, a few important limitations should also be addressed in future work.

Generating multiple custom concepts within the same image remains challenging. Simply prompting a merged model with multiple custom tokens usually leads to incoherent hybrids of both objects. Prior works~\cite{gu2023mix} have explored spatial guidance for better disentangling concepts in a single generation, and we have also employed similar techniques to generate our results. However, these methods still lead to failure cases as illustrated in Fig.~\ref{fig:failure}. Concepts are often ignored, or attributes can leak to neighboring concepts. Future work should aim to address these struggles to further enable multi-concept generations.

Storing individual LoRAs, even those trained with our method can also be expensive. Although LoRAs are already compressive due to their low-ranked nature, storing a large bank of concepts for modualr customization can still be expensive. Works such as SVDiff~\cite{han2023svdiff} takes steps towards further compressing LoRAs while maintaining fidelity of generated images. However, our method does not naturally fit in with the SVDiff method, implying the need for a tailored compressing methodology.
\newpage
\begin{figure*}[t!]
\centering
\includegraphics[width=0.83\linewidth]{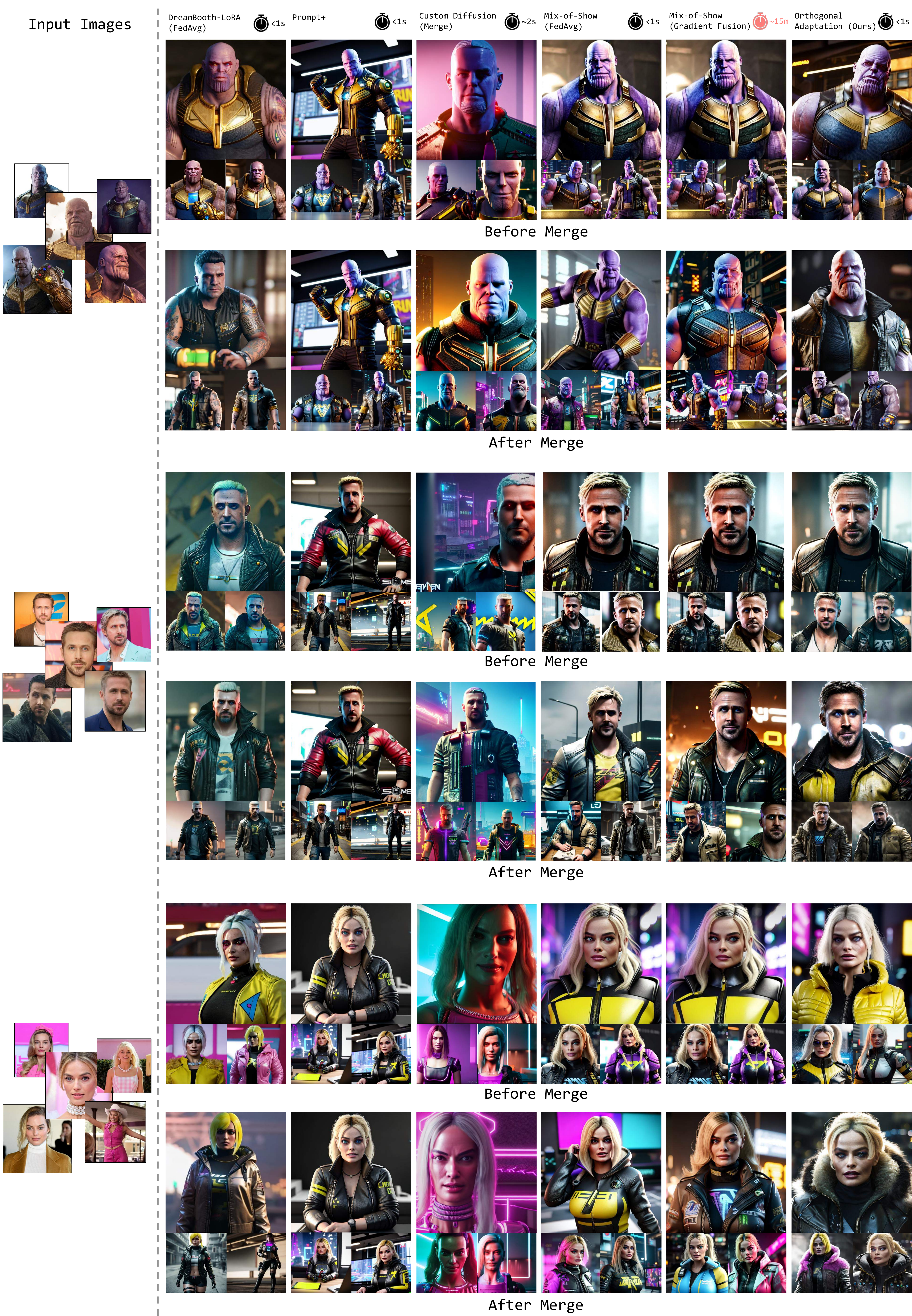}
\caption{\textbf{Extended multi-concept results.} We show results for each method before and after merging the individually trained models into a single, merged model. Our method is able to capture the target identity with high fidelity before and after the merging process, while keeping the merging process instantaneous.}
\label{fig:identity_supp}
\vspace{-1.5em}
\end{figure*}

\newpage
\begin{figure*}[t!]
\centering
\includegraphics[width=0.9\linewidth]{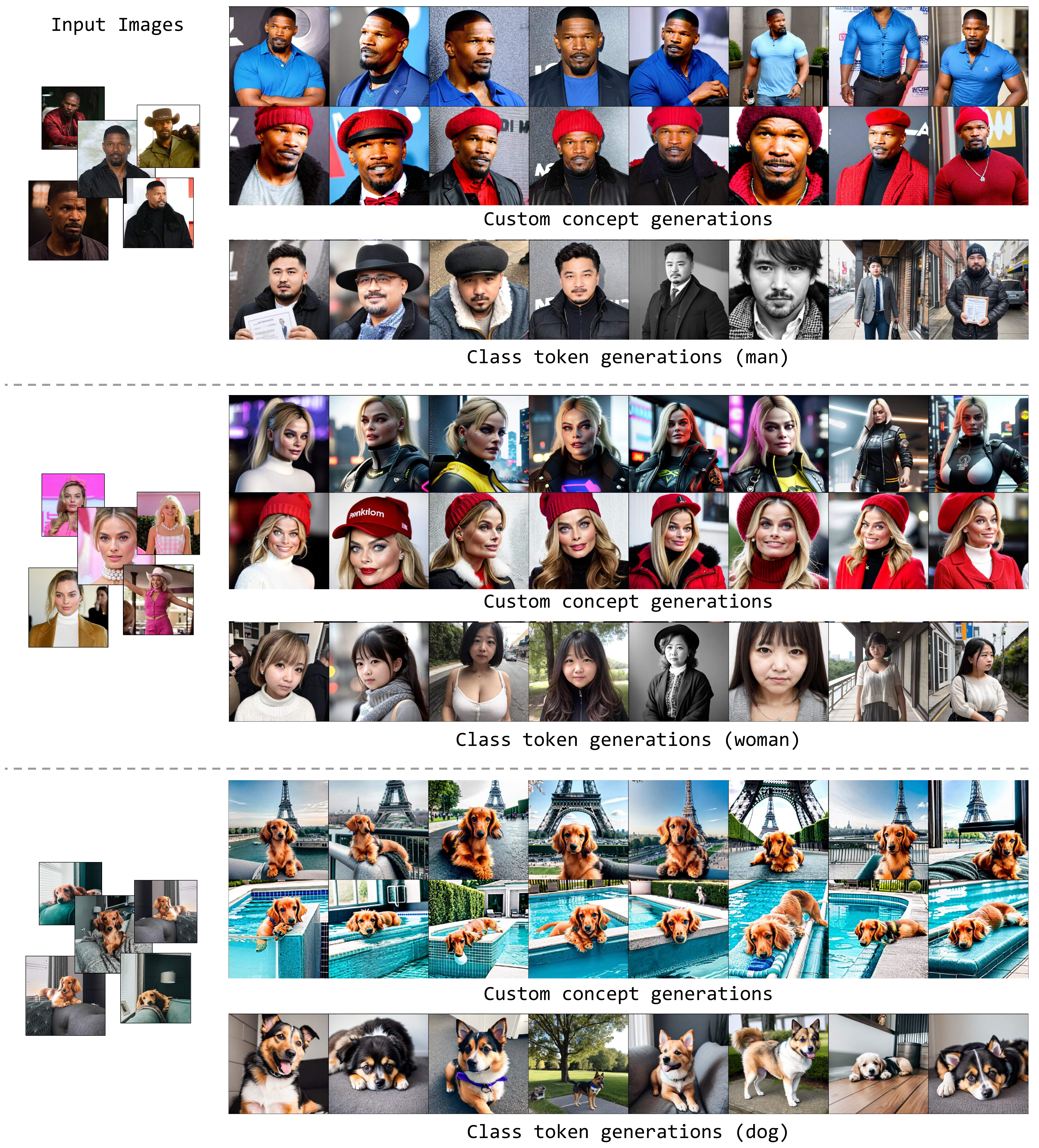}
\caption{\textbf{Preservation of class label.} Although our method does not enforce an explicit class preservation loss similar to prior works~\cite{ruiz2023dreambooth,kumari2022customdiffusion}, our method is able to preserve diversity when generating images of the class label used for initialization of the custom concept token. We show this across three different classes, namely: \textit{man}, \textit{woman}, and \textit{dog}.}
\label{fig:class}
\vspace{-1.5em}
\end{figure*}

\end{document}